\documentclass[twoside]{article}

\usepackage[preprint]{aistats2027}

\usepackage[T1]{fontenc}
\usepackage[utf8]{inputenc}

\usepackage{graphicx}

\usepackage{mathtools}
\usepackage{amssymb}
\usepackage{amsfonts}
\usepackage{amsthm}
\usepackage{centernot}
\usepackage{enumitem}

\usepackage{algorithm}
\usepackage{algpseudocode}
\usepackage{xurl}
\usepackage[hidelinks]{hyperref}
\hypersetup{
  pdftitle={Removing Information Content Does Not Certify Tamper Resistance in Open-Weight Models},
  pdfauthor={Domenic Rosati, Alessa Carbo, Ali Dadsetan, Hong Huang, Matthew Young, Subhabrata Majumdar, Frank Rudzicz, Hassan Sajjad}
}
\usepackage{cleveref}

\usepackage{booktabs}

\usepackage[round]{natbib}

\theoremstyle{plain}
\newtheorem{theorem}{Theorem}[section]
\newtheorem{proposition}[theorem]{Proposition}
\newtheorem{lemma}[theorem]{Lemma}

\theoremstyle{definition}
\newtheorem{definition}[theorem]{Definition}

\theoremstyle{remark}

\newtheorem{corollary}[theorem]{Corollary}

\begin{document}

\runningtitle{Removing Information Content Does Not Certify Tamper Resistance}
\runningauthor{Rosati et al.}

\twocolumn[

\aistatstitle{Removing Information Content Does Not Certify \\ Tamper Resistance in Open-Weight Models}

\aistatsauthor{
  Domenic Rosati\textsuperscript{1,4,*} \quad
  Alessa Carbo\textsuperscript{2} \quad
  Ali Dadsetan\textsuperscript{1,4} \quad
  Hong Huang\textsuperscript{1} \\[0.5ex]
  \bfseries Matthew Young\textsuperscript{1} \quad
  Subhabrata Majumdar\textsuperscript{3} \quad
  Frank Rudzicz\textsuperscript{1,4} \quad
  Hassan Sajjad\textsuperscript{1}
}

\aistatsaddress{
  \textsuperscript{1}\,Dalhousie University \\
  \textsuperscript{2}\,Johns Hopkins University \\
  \textsuperscript{3}\,Indian Institute of Management Bangalore \\
  \textsuperscript{4}\,Vector Institute \\[0.5ex]
  \textsuperscript{*}\,Corresponding author: \href{mailto:domenic.rosati@dal.ca}{\nolinkurl{domenic.rosati@dal.ca}}
} ]

\begin{abstract}
Does removing harmful information make open-weight models resistant to fine-tuning attacks? We show that mutual information at release alone cannot universally certify slow recovery. Function-preserving reparameterizations leave information unchanged while altering gradient-descent geometry, so an invariant certificate is bounded by the fastest reachable parameterization. We apply this principle to weight--data mutual information under training-data filtering and label--representation mutual information under capability removal. Training order can change recovery time at fixed weight--data information, while exact representation-level independence can preserve the entire parameter Jacobian. An explicit construction has both information quantities equal to zero and recovers in one gradient step. Controlled experiments illustrate order-dependent recovery and parameterization-dependent attack speed. These results identify the missing requirement for certification: constraints on attack dynamics beyond mutual information at release.
\end{abstract}

\section{Introduction}

Open-weight foundation models can be misused to produce deepfakes
\citep{maiberg2024deepfaketools} or assist with biological threats
\citep{karatzikos2026safeguarding}. Safeguards trained into these models can
be removed by fine-tuning \citep{qi2023fine}.
\emph{Tamper resistance} seeks to prevent safeguard removal when attackers
have access to model weights \citep{casper2026open}, but existing approaches
have shown limited robustness to adaptive fine-tuning attacks
\citep{qi2025evaluating,zloczower2026one}.

One proposed explanation is that safety interventions obscure harmful
capabilities without removing their underlying information
\citep{deeb2024unlearning,lucki2024adversarial,siddiqui2026dormant,hu2025unlearning}.
This motivates both filtering harmful data out of pretraining
\citep{obrien2025deepignorance,maini2026safety,korbak2023pretraining} and
removing acquired capabilities through unlearning
\citep{li2024wmdp,barez2025open}. The certification question is whether
successful information removal guarantees that an attacker will need many
updates to recover the target behavior.

We study \emph{information content} as statistical dependence, quantified
by mutual information (MI). This choice separates what a released model
reveals from how it changes under optimization. MI can be identical for
models with different parameter derivatives; those derivatives determine
how an attack changes the model's predictions. Our question is therefore
whether MI at release alone can lower-bound recovery time.

The argument proceeds in three steps. First, we define recovery time for a
specified attack and show that function-preserving reparameterizations
induce different gradient-descent preconditioners while keeping MI fixed
(\S\ref{sec:MI-orbit-geometry}). Any invariant certificate must respect the
fastest attack within the corresponding orbit. Second, we apply this
limitation to training-data filtering
(\S\ref{sec:training-data-filtering}). Excluding an independently
drawn harmful dataset gives zero weight--data MI, but does not constrain
transferable knowledge or optimization geometry. Even without filtering,
training order can attenuate a retained sample's behavioral effect while
preserving all its information in the weights. Third, we ask whether the
stronger condition of label--representation independence (the goal of unlearning) closes this gap
(\S\ref{sec:formalizing-capability-removal}). An explicit construction
achieves exact independence while preserving the entire parameter
Jacobian, and a special case recovers in one gradient step.

We test these predictions on the Deep Ignorance 6.9B strong-filter (SF) and unfiltered (UF) checkpoints.
We show that the tamper resistance results of \citet{obrien2025deepignorance} are simply artifacts of the parameterization chosen to represent the model.
A query--key (Q/K) gauge rescaling makes SF nearly match UF's training accuracy within 293 SGD updates (69.3\% for SF at $\alpha=8$ versus 68.9\% for UF at $\alpha=1$). Rescaling UF to $\alpha=1/24$ slows its median first attainment of 60\% training accuracy from 125 to 225 updates, compared with 150 for unscaled SF and 125 for rescaled SF.
The apparent tamper resistance of SF can thus be cheaply reproduced in UF by changing its parameterization
(\Cref{fig:qk-orbit-accuracy}). This indicates that safety pretraining may have little to no tamper-resistance effect once parameterization is taken into account.

\section{Formalizing Tamper Resistance}
\label{sec:tamper-resistance-def}

Tamper resistance concerns unauthorized modifications of a released model
\citep{che2025model}. We study attackers who fine-tune its weights on a
target dataset. Resistance to this attack family is necessary for a more
general tamper-resistance guarantee, and its update dynamics let us ask
precisely what information removal can certify.

Let $f_\theta$ be a model with parameters $\theta\in\mathbb R^d$ and let
$\mathcal D_f=\{(x_i,y_i)\}_{i=1}^N$ be the attack dataset. For a nonnegative
pointwise loss $\ell$, define
\[
\mathcal L_{\mathcal D_f}(\theta)
=\frac1N\sum_{i=1}^N\ell(f_\theta(x_i),y_i).
\]
An attack protocol $\mathcal A$ specifies the trainable parameters,
optimizer, step-size rule, and any stochastic sampling. Starting from the
released parameters $\theta$, its loss hitting time is
\[
\tau^\epsilon(\theta;\mathcal A)
:=\inf\{t\in\mathbb N_0:
\mathbb E_{\mathcal A}[\mathcal L_{\mathcal D_f}(\theta_t)]\leq\epsilon\},
\]
with $\inf\varnothing=\infty$. The expectation is over attack randomness,
conditional on the release and attack dataset. Unless stated otherwise,
we require initial $\mathcal L_{\mathcal D_f}(\theta)>\epsilon$, so success requires
at least one update.

\begin{definition}[Attack-relative and robust certificates]
\label{def:certificate}
Fix a class of releases and attack datasets, a threshold $\epsilon$, and
an attack family $\mathfrak A$. A certificate based on a release statistic
$\gamma(\theta)$ is a bound
\[
B(\gamma(\theta))\leq
\tau^\epsilon_{\mathfrak A}(\theta)
:=\inf_{\mathcal A\in\mathfrak A}\tau^\epsilon(\theta;\mathcal A)
\]
that holds throughout the specified class. For a singleton
$\mathfrak A=\{\mathcal A\}$ it is attack-relative; for a larger family it
must hold for every included attack. A bound exceeding one update is
nontrivial under the initial-loss condition above.
\end{definition}

The attack family may include momentum, adaptive step sizes, or
function-preserving preprocessing. We count optimization updates after
preprocessing; any separate preprocessing cost must be specified for a
compute-cost certificate. Retained utility can be imposed by restricting
the releases and successful attacks to a fixed utility-loss tolerance.
Our one-step construction below can preserve a separate utility predictor
exactly. The empirical studies specify their attack protocols and report loss or
accuracy hitting criteria.

For a random release $\Theta$, weight--data MI is a property of the data
sampling and training ``experiment,'' not of a single realized weight vector.
A certificate based on it must specify whether it holds for every release,
in expectation over releases, or with a stated probability. The
zero-information construction below has the same one-step hitting time
for every draw, so it applies to each of these conventions.

Prior work studies function-preserving curvature interventions and their
convergence guarantees under specified attack assumptions
\citep{rosati2026distribution,rosati2026limits}. Here we examine the
complementary question: can information content alone supply a lower
bound, without constraints linking it to the attack dynamics?

\section{Information vs.\ Optimization Geometry}
\label{sec:MI-orbit-geometry}

A certificate based only on MI assigns the same bound to every release
with the same information value. We first identify transformations that
preserve that value while changing how gradient descent acts. This
provides the common mechanism for the two information-removal settings
that follow.

Let $(S,X)$ be jointly distributed and let $Y_\theta=f_\theta(X)$.
We call a parameter symmetry group \emph{representation-compatible} when
it acts on the representation through an invertible output transformation,
as specified in Lemma~\ref{lem:decoupling}. A map is \emph{bimeasurable}
if it is bijective and both it and its inverse are measurable. The following lemma states the group conditions and their consequences
for information and derivatives.

\begin{lemma}[Mutual information--Jacobian decoupling]
\label{lem:decoupling}
Let a group $G$ act smoothly on the parameter and output spaces through
$T_g:\Theta\to\Theta$ and $\rho_g:\mathcal Y\to\mathcal Y$. Suppose
\[
f_{T_g(\theta)}(x)=\rho_g(f_\theta(x))
\qquad\text{for all }g,\theta,x.
\]
Then $I(S;Y_{T_g(\theta)})=I(S;Y_\theta)$. If $f$ is differentiable in
$\theta$, its parameter Jacobian satisfies
\[
D_\theta f_\theta(x)
=[D\rho_g(f_\theta(x))]^{-1}
D_{T_g(\theta)}f_{T_g(\theta)}(x)D_\theta T_g(\theta).
\]
Thus the Jacobian need not be invariant on the parameter orbit
$\mathcal O_G(\theta):=\{T_g(\theta):g\in G\}$.
\end{lemma}

\begin{proof}
Each smooth group action has a smooth inverse, so $\rho_g$ is
bimeasurable. Applying data processing to $\rho_g$ and its inverse proves
MI equality. Differentiating the equivariance identity gives the Jacobian
formula. Full proof details are given in the appendix.
\end{proof}

For an exactly function-preserving symmetry, $\rho_g$ is the identity.
Such symmetries include positive rescaling in homogeneous networks
\citep{dinh2017sharp} and query--key rescaling in Transformer attention
\citep{zhao2025symmetry}. Their effect on gradient descent is explicit.

\begin{proposition}[Orbit-induced preconditioning]
\label{prop:orbit-preconditioning}
Let $T_g(\theta)=A_g\theta$ be invertible and linear, with
$\mathcal L(A_g\theta)=\mathcal L(\theta)$ for every $\theta$.
Gradient descent initialized at $\theta'_0=A_g\theta_0$,
\[
\theta'_{t+1}=\theta'_t-\eta\nabla\mathcal L(\theta'_t),
\]
has pulled-back iterates $\theta_t=A_g^{-1}\theta'_t$ satisfying
\[
\theta_{t+1}=\theta_t-\eta P_g\nabla\mathcal L(\theta_t),
\qquad P_g=A_g^{-1}A_g^{-\top}.
\]
\end{proposition}

\begin{proof}
Differentiating loss invariance gives
$\nabla\mathcal L(A_g\theta)=A_g^{-\top}\nabla\mathcal L(\theta)$.
Substitute into the update and multiply by $A_g^{-1}$.
\end{proof}

For example, $A_\alpha=\operatorname{diag}(\alpha I,\alpha^{-1}I)$ induces
$P_\alpha=\operatorname{diag}(\alpha^{-2}I,\alpha^2I)$.
The initial function is unchanged, but the update metric varies with
$\alpha$ \citep{dinh2017sharp,neyshabur2015path,salimans2016weight}.
This identity holds throughout the trajectory for plain gradient descent;
other optimizers and restricted trainable blocks require their own analysis.
In particular, if the symmetry acts only on frozen parameters and leaves
the loss as a function of the trainable block identical, it leaves that
block's exact gradient unchanged.

\begin{corollary}[Orbit-wise certificate ceiling]
\label{cor:orbit-certificate}
Fix an attack $\mathcal A$ and a function-preserving symmetry group $G$.
If $\gamma(T_g(\theta))=\gamma(\theta)$ and
$B(\gamma(\theta))$ lower-bounds the attack hitting time for every
representative in $\mathcal O_G(\theta)$, then
\[
B(\gamma(\theta))\leq
\inf_{g\in G}\tau^\epsilon(T_g(\theta);\mathcal A).
\]
\end{corollary}

\begin{proof}
The same bound must hold for every $g$; take the infimum.
\end{proof}

The bound must hold even at the fastest representative. It can still be useful if every
representative is slow. If the attack family includes preprocessing
by a reachable subset $G_{\mathrm{acc}}\subseteq G$, its robust hitting
time is no larger than the corresponding infimum over $G_{\mathrm{acc}}$.
The realizable preconditioners are those induced by the model's actual
symmetries, rather than arbitrary positive-definite matrices.
\Cref{fig:neuron_orbit_mi_trace,fig:pythia-constant-on-orbit} illustrate
orbit-dependent recovery under fixed attacks
(Supp~\ref{app:additional-empirical-validation}).

The remaining question is whether removing information forces all such
representatives to be slow. We first examine data filtering, where exclusion
has a precise meaning in a sampling experiment.

\section{Training-Data Filtering and Training Order}
\label{sec:training-data-filtering}

Following information-theoretic learning analyses
\citep{xu2017information,pensia2018generalization,bu2020tightening,negrea2019information},
we model training as a channel from sampled data to learned parameters.
Draw safe and harmful candidate datasets independently,
\[
\mathcal D_S\sim P_S^{n_S},\qquad
\mathcal D_H\sim P_H^{n_H},\qquad
\mathcal D_S\perp\mathcal D_H.
\]
A filter $F_\rho$, with $\rho\in[0,1]$, removes a $\rho$-fraction of the
harmful samples, giving
$\mathcal D_{\mathrm{pre}}^\rho=\mathcal D_S\cup F_\rho(\mathcal D_H)$.
The endpoints are $F_0(\mathcal D_H)=\mathcal D_H$ and
$F_1(\mathcal D_H)=\varnothing$. The filter's random seed, any training
randomness, and the initialization are independent of $\mathcal D_H$;
the initialization may in particular be fixed foundation-model weights.
A fixed learning procedure, including its rule for ordering the retained
samples, produces
\[
\Theta^\rho=\mathcal A_{\Theta_0}(\mathcal D_{\mathrm{pre}}^\rho),
\qquad
\mathcal D_H\longrightarrow\mathcal D_{\mathrm{pre}}^\rho
\longrightarrow\Theta^\rho.
\]
Here $I(\mathcal D_H;\Theta^\rho)$ measures dependence on the particular
random candidate dataset. It does not measure all knowledge about the
harmful task: safe data and fixed pretrained features can support that
task while being independent of a fresh draw from $P_H$.

\begin{corollary}[Complete filtering leaves geometry unconstrained]
\label{cor:data-filtering}
At complete filtering, $I(\mathcal D_H;\Theta^1)=0$.
For every fixed bimeasurable parameter transformation $T_g$,
\[
I(\mathcal D_H;T_g(\Theta^1))=0.
\]
For a representation-compatible symmetry as in
Lemma~\ref{lem:decoupling}, this equality does not require the parameter
Jacobian to be invariant.
\end{corollary}

\begin{proof}
At $\rho=1$, training uses only $\mathcal D_S$, independently of
$\mathcal D_H$, so data processing gives zero MI. An invertible change of
variables preserves MI, while Lemma~\ref{lem:decoupling} gives the
Jacobian transformation.
\end{proof}

\paragraph{Deep Ignorance: recovery rankings depend on coordinates.}
\Cref{fig:qk-orbit-accuracy} examines the actual 6.9B strong-filter (SF)
and unfiltered (UF) checkpoints of \citet{obrien2025deepignorance} on
WMDP-bio, with three attack seeds and a budget of 293 SGD updates.
The multiplicative group $G=(\mathbb R_{>0},\times)$ acts on each
attention head by
$(W_Q,b_Q,W_K,b_K)\mapsto(\alpha W_Q,\alpha b_Q,\alpha^{-1}W_K,\alpha^{-1}b_K)$,
preserving every query--key dot product and therefore each checkpoint's
attention weights and released function in exact arithmetic.
This preprocessing requires a single pass over the Q/K tensors, with
no additional training or data.

\textbf{Rescaled SF reaches essentially the same training accuracy as UF.}
After 293 updates, SF at $\alpha=8$ reaches $69.3\pm4.4\%$
training accuracy, compared with $68.9\pm1.8\%$ for UF at $\alpha=1$;
unscaled SF reaches $58.5\pm5.0\%$.
\textbf{Slowing UF alone reverses its speed advantage over unscaled SF.}
Rescaling UF to $\alpha=1/24$ moves its median first logged attainment
of $60\%$ training accuracy from 125 to 225 updates, whereas SF in its
original coordinates takes 150 updates. Thus UF becomes slower than SF without excluding any
additional data or changing its released function. Its final held-out
accuracy falls from $60.0\pm2.2\%$ to $42.7\pm5.4\%$, also below
unscaled SF's $45.3\pm4.1\%$.
Against rescaled SF, the difference is larger: SF at $\alpha=8$ reaches
the same training threshold in a median of 125 updates and ends at
$69.3\pm4.4\%$ training and $54.0\pm5.0\%$ held-out accuracy,
compared with slowed UF's $61.5\pm7.6\%$ and $42.7\pm5.4\%$.

The apparent filtration advantage in this attack comparison can thus
disappear or reverse under inexpensive preprocessing that preserves
the released function of each checkpoint. This challenges interpreting
the recovery gap on the Deep Ignorance checkpoints as a barrier due
solely to information removed by filtering: the recovery ranking can
reflect the coordinates presented to the attack. This conclusion concerns
these runs and budget; full-model fine-tuning, including Q/K parameters,
allows the symmetry to change the attack's optimization geometry
(Supp~\ref{app:qk-orbit-accuracy-details}).

\begin{figure}[t]
    \centering
    \includegraphics[width=\linewidth]{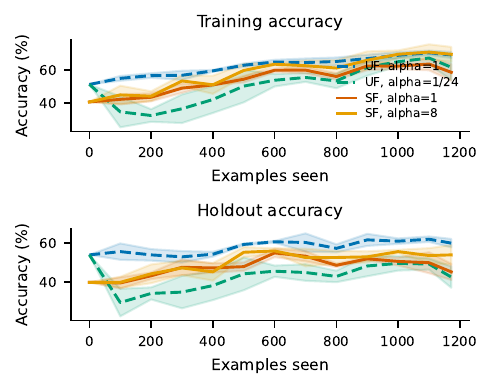}
    \caption{\textbf{Function-preserving rescaling erases and reverses the
    recorded filtration advantage.} After 293 updates, SF at $\alpha=8$
    nearly matches UF at $\alpha=1$ in training accuracy
    ($69.3\%$ versus $68.9\%$); rescaling UF to $\alpha=1/24$ alone
    slows its median first attainment of $60\%$ training accuracy
    from 125 to 225 updates, beyond unscaled SF's 150, and lowers its
    held-out accuracy from $60.0\%$ to $42.7\%$, below unscaled
    SF's $45.3\%$. Means and shaded standard deviations
    over three seeds; each transform preserves its checkpoint's released
    function in exact arithmetic. Protocol details are in
    Supp~\ref{app:qk-orbit-accuracy-details}.}
    \label{fig:qk-orbit-accuracy}
\end{figure}

Complete filtering therefore gives independence from the excluded sample
without fixing the geometry of subsequent learning. There is a
complementary issue at the other endpoint, $\rho=0$: later updates can
suppress a retained sample's behavioral effect. Examining this endpoint
separates deliberate exclusion from attenuation that training induces on
its own.

\subsection{Training order changes recovery at fixed information}
\label{sec:training-order}

Keep the same sampling experiment, with one harmful candidate
$\mathcal D_H=\{(x,Y)\}$, $Y$ uniform on $\{-1,+1\}$, and $T-1$ safe examples
whose labels are deterministically zero. The inputs can all be the same
fixed $x$; the safe sample is independent of $Y$. At $\rho=0$ every
example is retained. We vary only the position $i$ of the harmful sample,
leaving the sampled data, initialization, and update rule fixed.

\begin{proposition}[Ordering decouples sample information from recovery]
\label{prop:ordering-decoupling}
Consider $f_\theta(x)=\theta$ and
$\ell(\theta,y)=\frac12(\theta-y)^2$. Starting from $\theta_0=0$, train
on the sequence above for $T$ steps with constant step size
$\eta\in(0,1)$. If $Y$ appears at position $i\in\{0,\ldots,T-1\}$,
then
\[
\Theta_i=c_iY,\qquad c_i=\eta(1-\eta)^{T-i-1},
\]
and $I(\mathcal D_H;\Theta_i)=H(Y)$ for every $i$, whereas
$\ell(\Theta_i,Y)=\frac12(1-c_i)^2$.
Fine-tuning on the harmful example with step size $\gamma\in(0,1)$ gives
\[
\ell(\theta_s,Y)=(1-\gamma)^{2s}\ell(\Theta_i,Y).
\]
For a common threshold $0<\epsilon<\frac12(1-\eta)^2$, the hitting time is
\[
\tau_i^\epsilon=
\left\lceil
\frac{\log(\ell(\Theta_i,Y)/\epsilon)}{-2\log(1-\gamma)}
\right\rceil.
\]
Consequently, changing the sample's position can change recovery time
while preserving its weight--data MI.
\end{proposition}

\begin{proof}
The harmful update sets $\theta=\eta Y$; each subsequent safe update
multiplies it by $1-\eta$. Since $c_i>0$, the final parameter determines
$Y$ exactly, giving MI $H(Y)$. The fine-tuning residual is multiplied by
$1-\gamma$ at each step, yielding the loss and hitting-time formulas.
\end{proof}

Later safe updates can make $c_i$ arbitrarily small as their number grows.
The harmful sample then becomes difficult to detect through the released
prediction's magnitude, even though its label remains exactly recoverable
from the sign of the parameter. This is behavioral attenuation induced by
training order, not statistical removal of the sample's information.
At $\rho=1$ the parameter would instead remain zero, with zero MI.
The distinction is exact in this noiseless construction; quantization or
observation noise would define a different information channel.

Here every fine-tuning loss has the same curvature and contraction rate.
The hitting-time difference comes from the initial loss, illustrating
that MI determines neither the initial optimization state nor, by
\S\ref{sec:MI-orbit-geometry}, the update geometry. In deeper models,
composition of updates can affect both. To see the mechanism, an
infinitesimal perturbation $\delta\theta_{i+1}$ propagates to first order as
\[
\delta\theta_T\approx
(I-\eta_{T-1}H_{T-1})\cdots(I-\eta_{i+1}H_{i+1})
\delta\theta_{i+1},
\]
where $H_k$ is the per-step loss Hessian along the reference trajectory
\citep{wang2025capturing}. Later updates can contract or redirect an
earlier sample's effect. This links the scalar construction to training
trajectories without assuming uniform contraction in deep networks.

\Cref{fig:pythia-ordering} examines this dependence on Pythia-160M \citep{biderman2023pythia} with
AG News classification and E2E-NLG generation. The amount of target
training is fixed while its position in a utility-training stream varies.
These experiments test order-dependent recovery; exact equality of MI
across orders is established by the construction, not measured in the
language models (Supp~\ref{app:pythia-ordering-details}).

\begin{figure}[t]
\centering
\includegraphics[width=\linewidth]{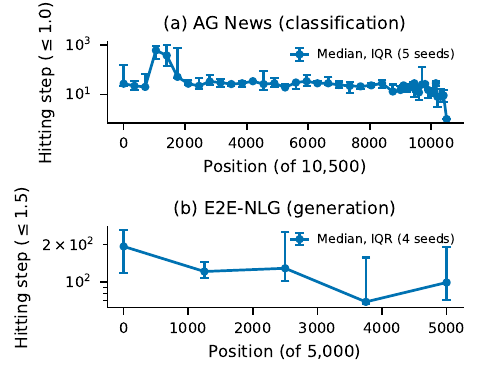}
\caption{\textbf{Target-block position changes recovery on Pythia-160M.}
(a) AG News classification \citep{zhang2015characterlevel};
(b) E2E-NLG generation \citep{novikova2017e2e}.
Recovery time trends downward as target data is introduced at a later position (see Supp~\ref{app:pythia-ordering-details}).
}
\label{fig:pythia-ordering}
\end{figure}

Filtering determines whether a sample enters training; ordering determines
how its effect survives later updates. The information measure used here
tracks training dependence on that sample. A certificate might instead require
that the released representation reveal nothing about the task labels.
We next test this stronger release condition.

\section{Capability Information Removal and Recovery}
\label{sec:formalizing-capability-removal}

Let $(X,Y)\sim P_{XY}$, with finite label space, and let
$Z_\theta=g_\theta(X)$ be the representation used by a downstream predictor.
The relevant requirement is now independence of $Z_\theta$ from the target
label $Y$. Following
information-bottleneck formulations
\citep{tishby2000information,saxe2018information} and their use in safety
interventions \citep{deeb2024unlearning,rosati2024repnoise}, we measure
retained target information by $I(Y;Z_\theta)$.

\begin{definition}[Complete capability information removal]
\label{def:complete-capability-removal}
A chosen representation achieves complete capability information removal
with respect to $Y$ if $I(Y;Z_\theta)=0$, equivalently $Y\perp Z_\theta$.
For log loss, this is equivalent to
\[
\inf_q-\mathbb E[\log q(Y\mid Z_\theta)]=H(Y),
\]
where the infimum is over all conditional decoders.
\end{definition}

The decoder characterization follows from
$I(Y;Z)=H(Y)-\inf_q\mathcal L_{\mathrm{CE}}(q)$
\citep{barber2004algorithm}. It is stronger than poor performance of a
particular head or probe: no decoder of the released representation can
beat the label-only baseline $H(Y)$. Poor performance by one predictor
may instead reflect information hiding if another decoder achieves
lower loss. Our constructions satisfy the stronger condition of complete
information removal.

Lemma~\ref{lem:decoupling} preserves $I(Y;Z)=0$ throughout any
representation-compatible orbit. By the certificate ceiling, an MI-only
bound at zero must therefore satisfy
\[
B(0)\leq\inf_{g\in G}\tau^\epsilon(T_g(\theta);\mathcal A).
\]
We now construct a zero-information release for which this bound is small.

\subsection{Exact removal with preserved tangent geometry}

Let $\phi(x)\in\mathbb R^p$ be fixed features, and consider the
linear-in-parameters representation $Z_{W,b}(x)=W\phi(x)+b\in\mathbb R^m$.
The features may come from a frozen pretrained backbone. Define the
ordinary removal objective
\[
\mathcal L_{\mathrm{remove}}(W,b)
=\frac12\mathbb E[\|Z_{W,b}(X)-u\|^2]
\]
for a constant independent target $u$.

\begin{proposition}[Perfect information removal with preserved Jacobian]
\label{prop:jacobian-preserving-removal}
For any pretrained $(W_{\mathrm{pre}},b_{\mathrm{pre}})$, the release
$W^\star=0$, $b^\star=u$ globally minimizes
$\mathcal L_{\mathrm{remove}}$ and satisfies $I(Y;Z_{W^\star,b^\star})=0$.
Using row-wise vectorization of $W$, its parameter Jacobian equals the
pretrained Jacobian at every input:
\[
D_{(W,b)}Z_{W,b}(x)
=\begin{bmatrix}I\otimes\phi(x)^\top&I\end{bmatrix}.
\]
Consequently, its tangent kernel on any fixed evaluation data is unchanged.
\end{proposition}

\begin{proof}
The release is constant, so its removal loss and label MI are zero.
The displayed Jacobian is independent of $W,b$, and hence so is its
Gram matrix $K=JJ^\top$.
\end{proof}

Thus this removal objective admits a solution with the entire pretrained
Jacobian intact. The next construction combines perfect capability removal
with recovery in one gradient step.

\begin{proposition}[Zero information with one-step recovery]
\label{prop:one-step}
Let $X$ be uniform on $\{-1,+1\}$, let $Y=X$, and let the target
representation and predictor be $Z_a(X)=aX$, released at $a_0=0$.
For any nonempty sample $\mathcal D_H$ from this distribution, squared
empirical loss is $\mathcal L_{\mathcal D_H}(a)=\frac12(a-1)^2$.
Then $I(Y;Z_{a_0})=0$, yet gradient descent with step size $1$ attains zero
loss in one update. For every $0<\epsilon<1/2$, its hitting time is one.

This release is also realizable in the complete-filtering experiment of
\S\ref{sec:training-data-filtering}, with
$I(\mathcal D_H;\Theta^1)=0$, while preserving a separate utility predictor.
Thus a universal MI-only lower bound over a class containing this example
cannot exceed one at zero information for any attack family containing
this gradient-descent protocol.
\end{proposition}

\begin{proof}
At release $Z_{a_0}=0$, independently of $Y$. Since $X^2=1$, every
sample has loss $\frac12(a-1)^2$, and the gradient at zero is $-1$;
one unit step gives $a_1=1$. For the filtering realization, use separate
parameter blocks $(a,v)$ for the target and utility predictors. Initialize
$a=0$ and train $v$ only on independent safe data, so
$\Theta^1=(0,V(\mathcal D_S))$ is independent of $\mathcal D_H$.
The target update changes only $a$, preserving the utility predictor.
The hitting time is one for every draw, which bounds any uniform,
expected, or high-probability certificate based on either zero MI value.
\end{proof}

This construction supplies the quantitative counterexample that the orbit
ceiling alone does not. It also makes the role of prior structure explicit:
the input feature already supports the task, although the released
representation carries no label information. Information removal need not
remove this capacity for subsequent learning.

An ordinary removal objective can also admit solutions with different
recovery behavior within the same model. For $Z_{a,b}(X)=abX$, both
$(a,b)=(0,1)$ and $(0,0)$ globally minimize
$\frac12\mathbb E[Z_{a,b}(X)^2]$ and have zero label MI. On the target
$Y=X$, one unit gradient step from $(0,1)$ gives $(1,1)$ and zero loss,
whereas $(0,0)$ has zero Jacobian and remains stationary under gradient
descent. The information value is identical at both releases; the selected
minimizer determines whether this gradient attack recovers. Unlearning's
choice of solution therefore matters beyond achieving MI zero.

\subsection{When surviving tangent geometry supports fast recovery}

The one-step example has ideal conditioning. For a general differentiable model, recovery
depends on how the tangent kernel acts on the current loss gradient, not
simply on whether the Jacobian is nonzero. Let $z(\theta)$ stack the attack
outputs, $J_t=D_\theta z(\theta_t)$, $K_t=J_tJ_t^\top$, and
$g_t=\nabla_z\mathcal L(z(\theta_t))$. Write
$\mathcal L^\star=\inf_\theta\mathcal L(\theta)$.

\begin{proposition}[Kernel alignment bounds recovery time]
\label{prop:pl-convergence-kernel-alignment}
Suppose along full-batch gradient descent that
\[
\tfrac12\|g_t\|^2\geq\alpha(\mathcal L(\theta_t)-\mathcal L^\star),
\qquad g_t^\top K_tg_t\geq\lambda\|g_t\|^2
\]
for constants $\alpha,\lambda>0$. Then the parameter loss satisfies the
Polyak--\L{}ojasiewicz (PL) inequality with $\mu=\alpha\lambda$.
If it is also $L$-smooth and $0<\eta\leq1/L$, $0<\eta\mu<1$, then
\[
\mathcal L(\theta_t)-\mathcal L^\star
\leq(1-\eta\mu)^t\Delta_0,
\quad\Delta_0=\mathcal L(\theta_0)-\mathcal L^\star.
\]
For $0<\delta<\Delta_0$, at most
\[
T_{\mathrm{suff}}(\delta)=
\left\lceil\frac{\log(\Delta_0/\delta)}{-\log(1-\eta\mu)}\right\rceil
\]
steps suffice to reach excess loss at most $\delta$.
\end{proposition}

\begin{proof}
The chain rule gives
$\|\nabla_\theta\mathcal L\|^2=g_t^\top K_tg_t
\geq2\alpha\lambda(\mathcal L-\mathcal L^\star)$.
Smooth descent then yields the geometric rate; solving for $t$ gives the
iteration bound.
\end{proof}

For squared loss $\frac12\|z-y\|^2$ with a realizable target,
$\alpha=1$ and $\mathcal L^\star=0$. In the linear construction,
$J$ and $K$ are constant throughout training. If $K\succeq\lambda I$,
the rate therefore applies despite exact zero information at release.
For the absolute-loss threshold in Definition~\ref{def:certificate}, use
$\delta=\epsilon-\mathcal L^\star>0$. This is an upper bound on attack
time, rather than a lower-bound certificate of resistance.

\paragraph{Controlled zero-information releases.}
\Cref{fig:pythia-capability-removal} implements the linear construction with
a factorized head on a frozen Pythia-160M backbone. The representation is
$z=A(sB_{\mathrm{source}}h)$, released with $A=0$; recovery trains only $A$.
All scales have $I(Y;Z)=0$, while $J_A^{(s)}=sJ_A^{(1)}$ and
$K_A^{(s)}=s^2K_A^{(1)}$. For $s>0$, the change of variables
$\widetilde A=sA$ makes SGD at rate $\eta$ exactly equivalent to the
$s=1$ attack at effective rate $\eta s^2$. Conversely, choosing
$\eta/s^2$ restores the reference trajectory. The experiment illustrates fast recovery from exact independence and
parameterization-dependent attack speed under cross-entropy loss
(Supp~\ref{app:pythia-capability-removal-details}).

\begin{figure}[t]
\centering
\includegraphics[width=\linewidth]{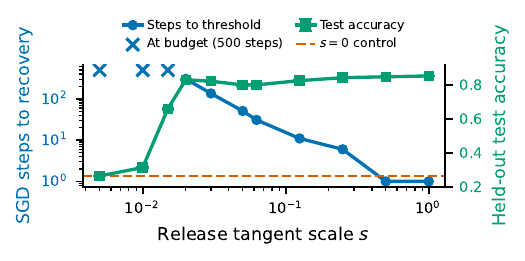}
\caption{\textbf{Recovery varies with tangent scale at identical
zero-information releases.} A factorized head on a frozen Pythia-160M
backbone has $I(Y;Z)=0$ at every scale $s$. Under fixed-rate SGD,
recovery is faster and held-out accuracy generally higher at larger
scales. For $s>0$, the effective learning rate is $\eta s^2$;
non-crossings are right-censored at the 500-step budget.}
\label{fig:pythia-capability-removal}
\end{figure}

\paragraph{Recovery after a weight-only scrub.}
The BERT-Tiny experiment in \Cref{tab:bert-retain-mse} tests the orbit
mechanism after an unlearning intervention. Scrubbing reduces target
accuracy and decoder-estimated predictive $\mathcal V$-information
\citep{xu2020usable} while preserving retain
accuracy. In a paired seed-0 comparison, an exact Q/K transformation of
the scrubbed checkpoint reduces recovery from 280 to 25 steps, with the
same attack batches and learning-rate grid. The transformed checkpoint
has the same released function and representations, so the recovery
difference cannot be attributed to a change in their information content.
The three-seed scrub mean and the single-seed orbit comparison are
reported separately (Supp~\ref{app:bert-retain-mse}).

\begin{table}[t]
\centering
\footnotesize
\resizebox{\linewidth}{!}{%
\begin{tabular}{lccc}
\toprule
Arm & Target acc. & $\widehat I_{\mathcal V}$ (L2) & Recovery \\
\midrule
Target-adapted & 91.7\% & $0.457$ & $0$ \\
No target adaptation & $\approx50\%$ & -- & $75\pm5$ \\
MSE scrub + retain & $53.4\pm3.6\%$ & $0.115$ & $252\pm67$ \\
\quad Seed 0, $\alpha=1$ & -- & -- & $280$ \\
\quad Seed 0, $\alpha=32$ & (unchanged) & (unchanged) & $25$ \\
Random initialization & 50.0\% & $0.140$ & $>1000$ \\
\bottomrule
\end{tabular}%
}
\caption{\textbf{Function-preserving reparameterization accelerates recovery
from a scrubbed BERT checkpoint.} The paired seed-0 comparison is
$280\to25$ steps. Scrub and no-target-adaptation summaries are means
$\pm$ standard deviations over three seeds. Information entries are final-encoder-layer logistic-decoder estimates
of predictive $\mathcal V$-information \citep{xu2020usable} in nats; all three layers are reported in the supplement.
The no-target-adaptation control has lower retain accuracy than the
scrubbed model; full results are in Supp~\ref{app:bert-retain-mse}.}
\label{tab:bert-retain-mse}
\end{table}

\section{Discussion}

The two information notions we study fail to constrain recovery for related but
distinct reasons. Weight--data MI tracks dependence on a sampled dataset,
so excluding that draw need not remove transferable task structure.
Label--representation MI concerns what a fixed representation reveals,
so even exact independence leaves its parameter derivatives unconstrained.
Training order additionally changes the starting loss without necessarily
changing the information encoded in the weights. These distinctions
explain why information at release does not determine learning dynamics.

The result is a limit on standalone certification, not on the usefulness
of unlearning or filtering. Either intervention can change the initial
loss, reachable representations, or update geometry in ways that slow
an attack. A guarantee requires assumptions connecting those changes to
attack dynamics; MI zero alone allows both the fast-recovery and
degenerate-kernel constructions above.

The orbit ceiling gives a complementary evaluation requirement. When an
attacker can cheaply change parameterization, a robust lower bound must
hold at the fastest reachable representative. Reporting the released
coordinates, trainable blocks, optimizer, and tuning budget therefore
matters for interpreting recovery comparisons. The BERT comparison and
the exact learning-rate equivalence in the factorized-head experiment
illustrate two ways that this requirement can affect an evaluation.

Geometry offers a direct language for the missing assumptions, but a
recovery bound and a resistance certificate require different arguments.
Our PL analysis lower-bounds progress to upper-bound attack time. A
resistance certificate would instead need to bound attainable progress
throughout the allowed attack trajectories. Tangent-kernel scale and
alignment may contribute to such bounds, together with step-size limits,
control of nonlinear remainders, and the evolution of $K_{\theta_t}$.
A kernel statistic at release alone does not supply these controls when
features change during training. A broader information-theoretic quantity
might incorporate tangent geometry. Fisher information captures local
parameter sensitivity: it defines a metric on infinitesimal changes to a
parameterized distribution, rather than statistical dependence between
two random variables. Whether such local geometry can support a global
notion of information content remains open.

Further questions concern these controls under reparameterization-invariant
attacks such as natural gradient, during feature learning, or in the
presence of global constraints such as barriers or mode connectivity
\citep{rosati2026limits,siddiqui2026dormant,cheng2025understanding}.
These training-time results leave inference-time information bounds intact.

\section{Conclusion}

Mutual information at release alone cannot universally certify slow
fine-tuning recovery. Exact exclusion of an independently sampled dataset
and exact label--representation independence both admit a one-step
recovery example, while function-preserving symmetries and training order
show how recovery can vary at fixed information. Information removal can
still be useful, but a tamper-resistance guarantee must also constrain
the dynamics of the attack it is intended to resist.

\subsection*{AI use statement}

The initial draft was written without generative AI tools. Generative AI
tools assisted with experiment design and execution, revision for clarity
and cohesion, and checking of mathematical
constructions. The authors are responsible for verifying the final text,
mathematical claims, and experimental artifacts.

\bibliographystyle{plainnat}
\bibliography{bibliography}

%%%%%%%%%%%%%%%%%%%%%%%%%%%%%%%%%%%%%%%%%%%%%%%%%%%%%%%%%%%%
\section*{Checklist}

\begin{enumerate}

  \item For all models and algorithms presented, check if you include:
  \begin{enumerate}
    \item A clear description of the mathematical setting, assumptions,
    algorithm, and/or model. [Yes -- Sections~\ref{sec:tamper-resistance-def}--\ref{sec:formalizing-capability-removal} state the channel/orbit setup and every construction explicitly.]
    \item An analysis of the properties and complexity (time, space, sample
    size) of any algorithm. [Yes -- hitting-time bounds are given in Prop.~\ref{prop:ordering-decoupling} and Prop.~\ref{prop:pl-convergence-kernel-alignment}.]
    \item (Optional) Anonymized source code, with specification of all
    dependencies, including external libraries. [Code to be released upon publication.]
  \end{enumerate}

  \item For any theoretical claim, check if you include:
  \begin{enumerate}
    \item Statements of the full set of assumptions of all theoretical
    results. [Yes.]
    \item Complete proofs of all theoretical results. [Yes -- every lemma, proposition, and corollary has a complete proof in the main text.]
    \item Clear explanations of any assumptions. [Yes.]
  \end{enumerate}

  \item For all figures and tables that present empirical results, check if
  you include:
  \begin{enumerate}
    \item The code, data, and instructions needed to reproduce the main
    experimental results (either in the supplemental material or as a
    URL). [Code and data to be released upon publication.]
    \item All the training details (e.g., data splits, hyperparameters, how
    they were chosen). [Yes -- main-text experimental protocols are in Supp~\ref{app:pythia-ordering-details}--\ref{app:bert-retain-mse}; the supplementary orbit protocols are in Supp~\ref{app:additional-empirical-validation}.]
    \item A clear definition of the specific measure or statistics and
    error bars (e.g., with respect to the random seed after running
    experiments multiple times). [Yes -- main-text studies specify seed counts and summaries; supplementary orbit plots report single-run protocols.]
    \item A description of the computing infrastructure used (e.g., type of
    GPUs, internal cluster, or cloud provider). [Partial -- the execution backends are described in Supp~\ref{app:compute-infrastructure}; exact Pythia GPU models were not recorded here.]
  \end{enumerate}

  \item If you are using existing assets (e.g., code, data, models) or
  curating/releasing new assets, check if you include:
  \begin{enumerate}
    \item Citations of the creator if your work uses existing assets. [Yes -- pretrained models and datasets are cited at first use (Pythia, BERT, AG News, E2E-NLG, Fashion-MNIST).]
    \item The license information of the assets, if applicable. [No -- asset citations are provided, but individual licenses are not enumerated.]
    \item New assets either in the supplemental material or as a URL, if
    applicable. [Not Applicable -- no new datasets or model checkpoints are released.]
    \item Information about consent from data providers/curators. [Not Applicable -- only public benchmark datasets are used.]
    \item Discussion of sensible content if applicable, e.g., personally
    identifiable information or offensive content. [Not Applicable -- the reported experiments use public classification, generation, and synthetic association tasks; no new personal data were collected.]
  \end{enumerate}

  \item If you used crowdsourcing or conducted research with human
  subjects, check if you include:
  \begin{enumerate}
    \item The full text of instructions given to participants and
    screenshots. [Not Applicable.]
    \item Descriptions of potential participant risks, with links to
    Institutional Review Board (IRB) approvals if applicable. [Not Applicable.]
    \item The estimated hourly wage paid to participants and the total
    amount spent on participant compensation. [Not Applicable.]
  \end{enumerate}

\end{enumerate}

\clearpage
\appendix
\thispagestyle{empty}

% Supplementary material: To improve readability, you must use a single-column format for the supplementary material.
\onecolumn
\section{Additional Empirical Validation}
\label{app:additional-empirical-validation}

The two toy-scale orbit demonstrations referenced in the main text are
shown here as supplementary illustrations.

\begin{figure}[ht]
    \centering
    \includegraphics[width=0.55\linewidth]{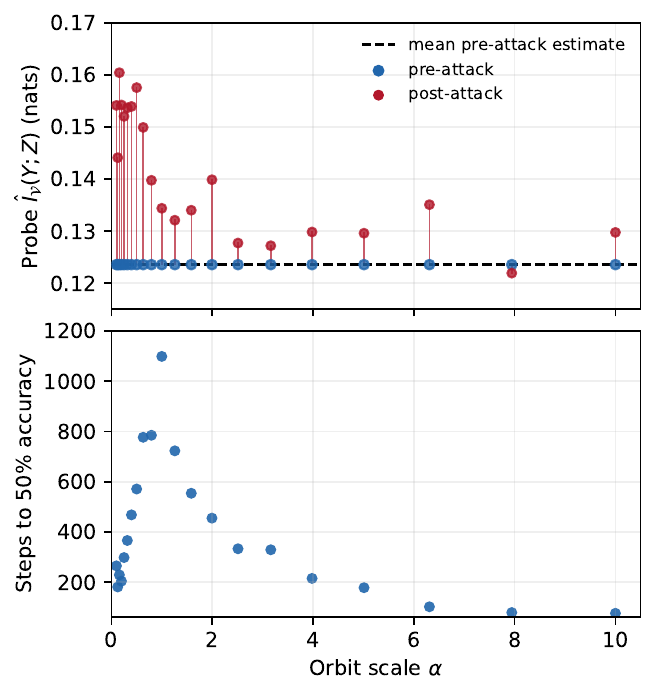}
    \caption{\textbf{Recovery varies along a function-preserving orbit.}
    A Fashion-MNIST MLP \citep{xiao2017fashionmnist} is rescaled by
    $\alpha$. The upper panel shows held-out logistic-probe
    $\widehat I_{\mathcal V}(Y;Z)$ \citep{xu2020usable} before and after the attack;
    the dashed line is the mean pre-attack estimate. Exact invariance
    of the released function and representation follows from the symmetry,
    rather than from these probe estimates. The lower panel reports
    recovery to 50\% target-class accuracy under momentum SGD.}
    \label{fig:neuron_orbit_mi_trace}
\end{figure}

\begin{figure}[ht]
    \centering
    \includegraphics[width=0.55\linewidth]{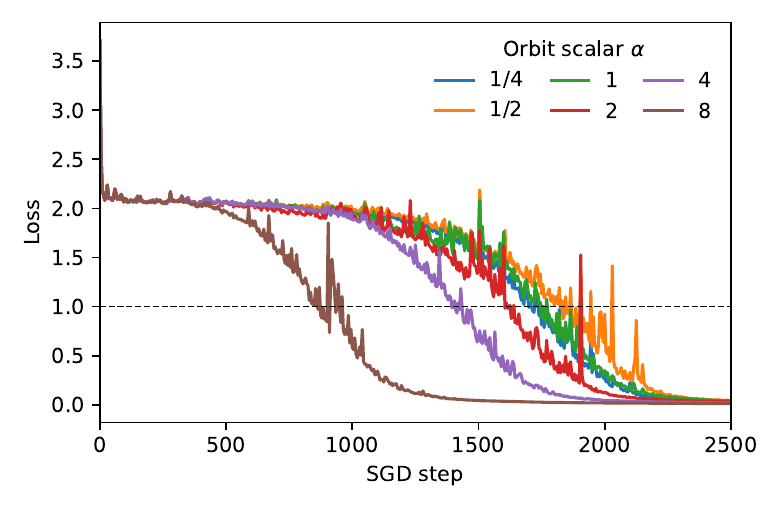}
    \caption{Function-preserving reparameterization changes relearning time. Pretrained Pythia-14M \citep{biderman2023pythia} models are transformed along the exact query–key
  attention orbit $((W_Q,b_Q)\mapsto\alpha(W_Q,b_Q))$, $((W_K,b_K)\mapsto\alpha^{-1}(W_K,b_K))$, leaving all attention scores—and therefore
  $I(X;Y_\theta)$—constant by construction. Starting from identical outputs, each model relearns the same 128 random item–token associations
  using vanilla SGD with an identical minibatch sequence. The trajectories cross the common (L=1) threshold at substantially different steps,
  demonstrating that optimization time can vary along an information-invariant parameter orbit.}
    \label{fig:pythia-constant-on-orbit}
\end{figure}

\paragraph{Fashion-MNIST protocol.}
A two-hidden-layer ReLU MLP of width 64 is trained for 500 updates on the
non-target classes among the first 10{,}000 training images (target class
6 excluded; SGD rate $0.03$, momentum $0.9$, batch size 32, sampling seed
2). A single source model is rescaled at 21 logarithmically spaced values
$\alpha\in[0.1,10]$, multiplying the first hidden layer's weights and
bias by $\alpha$ and dividing the next layer's weights by $\alpha$.
The maximum measured logit deviation across the sweep is
$1.82\times10^{-5}$. Recovery uses all ten classes, momentum SGD at rate
$0.001$, momentum $0.9$, and the same minibatch stream (seed 77) for
2{,}000 updates. The endpoint is 50\% accuracy on 32 held-out class-6
images. Predictive $\mathcal V$-information probe estimates
\citep{xu2020usable} use penultimate representations of the first
2{,}000 test images, binary class-6 labels, a 70/30 probe fit/evaluation
split (seed 0), and PCA to at most 32 dimensions. Logistic regularization
is selected on an inner validation split. Each orbit point is one run;
there are no between-run error bars.

\paragraph{Pythia-14M protocol.}
A single seed (17) generates 128 item-to-token associations over eight
answer tokens. All parameters are fine-tuned in FP32 with vanilla SGD
(rate $10^{-5}$, batch size 8), using identical minibatch orders after each
Q/K rescaling. Loss is evaluated every five updates on the same 128
associations, so the experiment measures memorization rather than
held-out generalization. \Cref{fig:pythia-constant-on-orbit} displays six
scales $\{1/4,1/2,1,2,4,8\}$ and a common loss threshold of 1; their
hitting times range from 865 to 1{,}825 steps. The 5{,}000-step runs have
zero measured initial logit deviation under the tested transforms.

\section{Experimental Details}
\label{app:compute-infrastructure}

The following protocols support the four main-text empirical results:
the Deep Ignorance gauge-transform witness, the Pythia ordering sweep, the exact-zero factorized-head construction,
and the BERT scrub/orbit comparison. They test distinct implications of
the theory: order-dependent recovery, recovery from an information-free
representation under a specified attack, and recovery variation at a fixed
released function. Exact weight--data MI equality across training orders
and one-step recovery under complete filtering are established by the
mathematical constructions.

\paragraph{Compute infrastructure.}
The Fashion-MNIST MLP and BERT-Tiny experiments were run locally on a
MacBook (Apple MPS/CPU). The Pythia-160M experiments used a single
GPU; the Deep Ignorance 6.9B experiments used a single NVIDIA H100 GPU
(80GB). No experiment reported here required multi-GPU or multi-node training.

\subsection{Training data filtration experiments}
\label{app:qk-orbit-accuracy-details}

For every attention head, replacing $(W_Q,W_K)$ by $(\alpha W_Q,\,W_K/\alpha)$
leaves the attention logits $W_Q^\top W_K$, and therefore the entire
function the network computes, exactly unchanged: this is a gauge
transformation of the checkpoint, not a different model.
\Cref{fig:qk-orbit-accuracy} reports WMDP-bio
\citep{li2024wmdp} results from two Deep Ignorance checkpoints
\citep{obrien2025deepignorance}: SF at $\alpha=8$ has 8.7 percentage
points higher final held-out accuracy than SF at $\alpha=1$, while UF
at $\alpha=1/24$ has 17.3 points lower accuracy than UF at $\alpha=1$.

\paragraph{Models and gauge transform.}
The two checkpoints are \url{EleutherAI/deep-ignorance-unfiltered} (UF) and
\url{EleutherAI/deep-ignorance-e2e-strong-filter} (SF). For a chosen $\alpha>0$, every
attention head's query and key projection rows are rescaled as $W_Q\mapsto\alpha W_Q$,
$W_K\mapsto W_K/\alpha$ (biases likewise), applied once to all layers before the attack
begins. Since attention depends on $W_Q$ and $W_K$ only through $W_Q^\top W_K$, this is
exact for any $\alpha$: the maximum absolute difference between pre- and post-transform
logits on a held-out batch is at machine precision in every run.

\paragraph{Task and data split.}
The task is four-way multiple-choice WMDP-bio question answering, scored by cross-entropy
over the four answer-letter logits. From the fixed pool of 1,273 items, a permutation seeded
once (independent of the run seed) sets aside 100 held-out test items; the remaining 1,173
are the training pool.

\paragraph{Attack.}
The attack is full-model fine-tuning: all model parameters, including the
query and key projections, are trained by plain SGD (learning
rate $2.5\times10^{-4}$) for one epoch over the training pool (293 steps, batch size 4,
$293\times4=1{,}172$ examples seen). Three seeds (17, 18, 19) vary only the training
minibatch order; the test split, the gauge transform, and the trainable parameter set are
identical across seeds. Train and held-out test accuracy are logged every 25 steps.

\paragraph{Conditions and reported quantities.}
Four (arm, $\alpha$) conditions are run: UF at $\alpha=1$ and $\alpha=1/24$, SF at
$\alpha=1$ and $\alpha=8$. The figure plots the mean and $\pm1$ standard deviation across
the 3 seeds of each condition's training and held-out accuracy against examples seen.
Final (step-293) held-out accuracy: UF $\alpha=1$, $60.0\pm2.2\%$; UF $\alpha=1/24$,
$42.7\pm5.4\%$; SF $\alpha=1$, $45.3\pm4.1\%$; SF $\alpha=8$, $54.0\pm5.0\%$.
Final training accuracy in the same condition order is
$68.9\pm1.8\%$, $61.5\pm7.6\%$, $58.5\pm5.0\%$, and
$69.3\pm4.4\%$ (mean $\pm$ population standard deviation across seeds).
For a descriptive comparison of recovery speed, the first logged
training accuracy of at least $60\%$ occurs at steps
$\{125,100,125\}$ for UF at $\alpha=1$,
$\{150,175,150\}$ for SF at $\alpha=1$,
$\{125,125,150\}$ for SF at $\alpha=8$, and
$\{225,225,250\}$ for UF at $\alpha=1/24$, giving medians of 125,
150, 125, and 225 updates, respectively. Accuracy is evaluated
every 25 updates and need not remain above the threshold thereafter.

\subsection{Experimental details: Pythia ordering, two target tasks}
\label{app:pythia-ordering-details}

\Cref{fig:pythia-ordering} varies the position of a fixed target-training
block within a utility stream. Recovery time trends downward as the
target block is introduced later: AG News shows a sharp late-stream drop,
while E2E-NLG shows a smaller dip at position 3{,}750. These experiments hold
exposure fixed; they do not measure weight--data MI or establish its
invariance across orders.

Both experiments share the base model (\texttt{EleutherAI/pythia-160m},
bf16) and construction: replay $p$ utility (C4 \citep{raffel2020exploring}, 256-token blocks) updates,
train a target block, replay the remaining utility updates to a fixed
total, then attack the released model with fresh SGD/AdamW on held-out
target examples to a fixed validation-loss threshold (hitting time =
first step at or below threshold; a seed that never crosses within budget
is censored, not dropped).

\paragraph{AG News (panel a).} Target is 4-class AG News; utility is
10{,}500 SGD (clipped) updates; target block is 2{,}000 steps;
recovery is a frozen-body, linear-head-only attack for 2{,}000 steps at
lr $10^{-2}$, threshold loss $\le1.0$. 31 positions (step 350) plus a
16-position zoomed follow-up (step 100) over the last 1{,}500 steps, 5
seeds each.

\paragraph{E2E-NLG (panel b).} Target is E2E NLG realization-token
causal-LM loss; utility totals 5{,}500 AdamW updates; target block is
16{,}000 steps; recovery is full-parameter AdamW for 1{,}500 steps at lr
$3\times10^{-5}$, threshold loss $\le1.5$. 5 positions
$\{0,1250,2500,3750,5000\}$, 4 seeds, all passing a pre-registered
acquisition-invariance gate (post-target loss spread $\le0.05$ nats across
positions within a seed -- at 8{,}000 target steps only 2 of 4 seeds
passed, so target steps were doubled to meet the same acquisition criterion for all four seeds).

\subsection{Capability removal experimental details}
\label{app:pythia-capability-removal-details}

\paragraph{Model and source acquisition.}
A frozen \texttt{EleutherAI/pythia-160m} backbone supplies its final-token
hidden state $h$. The four-class AG News head consists of
$v=Bh$, $z=Av+b$, and a linear classifier $C$, with
$B\in\mathbb R^{64\times d}$, $A\in\mathbb R^{64\times64}$,
and $b\in\mathbb R^{64}$. A source head
$(B_{\mathrm{source}},A_{\mathrm{source}},b_{\mathrm{source}},C)$ is trained
once on the training split with AdamW (learning rate $10^{-3}$, weight
decay $10^{-4}$, batch size 64, 1{,}000 steps, $A$ initialized to the
identity). The backbone is never updated. All recovery conditions share
this source head.

\paragraph{Release and recovery.}
For each scale
\[
s\in\{1,0.5,0.25,0.125,0.0625,0.05,0.03,0.02,0.015,0.01,
0.005,0.001,10^{-5},0\},
\]
set $B_s=sB_{\mathrm{source}}$, $A=0$, $b=0$, and freeze $B_s,b,C$.
The release has $z=0$ for every input, hence $I(Y;Z)=0$ exactly.
Released logits are bit-identical across all fourteen scales.
Recovery trains only $A$ with vanilla SGD (learning rate $0.1$, batch
size 64, 500 steps). The rate is chosen at $s=1$ and held fixed across
scales. Five seeds (0--4) vary the minibatch order, with each seed's
stream replayed across all scales.

\paragraph{Exact scaling law.}
Since the trainable matrix enters through $A(sB_{\mathrm{source}}h)$,
\[
J_A^{(s)}=sJ_A^{(1)},\qquad
K_A^{(s)}=s^2K_A^{(1)},\qquad
q_A(s):=\|\nabla_A\mathcal L\|^2=s^2q_A(1)
\]
at release. The gradient-energy ratio agrees with $s^2$ to five significant
figures over the tested positive scales, down to $s=10^{-5}$ and ratio
$10^{-10}$ (\Cref{tab:pythia-capability-removal-qa}).
For $s>0$, setting $\widetilde A=sA$ gives
$\mathcal L_s(A)=\mathcal L_1(\widetilde A)$, so an update at rate $\eta$
induces
\[
\widetilde A_{t+1}=\widetilde A_t-
\eta s^2\nabla_{\widetilde A}\mathcal L_1(\widetilde A_t).
\]
This is an exact minibatch-by-minibatch equivalence in real arithmetic.
Choosing $\eta_s=0.1/s^2$ restores the $s=1$ reference trajectory;
small positive scales therefore create no intrinsic recovery barrier.
At $s=0$, the trainable block's Jacobian vanishes and SGD cannot update
$A$. This is a zero-gradient control for this attack, not a claim of
resistance to unrestricted weight modifications.

\paragraph{Data and endpoints.}
Shuffle AG News's official training file once with seed 0 and select
4{,}096 training and 4{,}096 validation examples. The official test file
supplies 2{,}048 held-out examples. Test accuracy is evaluated only at the
end of each recovery run. The hitting criterion is validation accuracy
at least $0.7$, chosen using the $s=1$ pilot trajectory and fixed across
all scales and seeds. Validation is evaluated after the first update and
then every five updates. Runs that do not cross within 500 steps are
right-censored ($\tau>500$); plotted budget markers are not observed hits.

\Cref{fig:pythia-capability-removal} reports median observed hitting times
and mean held-out accuracy with one standard deviation across seeds.
All seeds are censored at the scales marked as budget-limited, as detailed
in \Cref{tab:pythia-capability-removal-qa}. The experiment tests exact
release independence and scale-dependent recovery, rather than the
trajectory-wide PL assumptions for cross-entropy loss.

\begin{table}[ht]
\centering
\caption{Multi-seed recovery summary (5 seeds; hitting threshold
val.\ accuracy \(\ge0.7\); 500-step budget). \(q_A(1)=6.358\)
(mean at step 0, seed-invariant since the released state does not depend on
seed); the predicted ratio column is \(s^2\). For \(s>0\), attacking at
nominal learning rate \(0.1\) is exactly equivalent to attacking the
\(s=1\) parameterization at effective learning rate \(0.1s^2\).}
\label{tab:pythia-capability-removal-qa}
\begin{tabular}{@{}rrrrrr@{}}
\toprule
$s$ & hit (med/min/max) & censored & test acc.\ (mean, sd) & $q_A(s)/q_A(1)$ & $s^2$ \\
\midrule
1        & 1 / 1 / 6       & 0/5 & 0.8530 (0.0017) & 1.000      & 1.000 \\
0.5      & 1 / 1 / 1       & 0/5 & 0.8481 (0.0012) & 0.2500     & 0.2500 \\
0.25     & 6 / 6 / 6       & 0/5 & 0.8427 (0.0008) & 0.06250    & 0.06250 \\
0.125    & 11 / 11 / 11    & 0/5 & 0.8254 (0.0018) & 0.01563    & 0.01563 \\
0.0625   & 31 / 31 / 36    & 0/5 & 0.8006 (0.0026) & 0.003906   & 0.003906 \\
0.05     & 51 / 51 / 51    & 0/5 & 0.8001 (0.0036) & 0.002500   & 0.002500 \\
0.03     & 136 / 131 / 141 & 0/5 & 0.8219 (0.0039) & 0.0009000  & 0.0009000 \\
0.02     & 306 / 301 / 311 & 0/5 & 0.8309 (0.0013) & 0.0004000  & 0.0004000 \\
0.015    & $>500$       & 5/5 & 0.6603 (0.0062) & 0.0002250  & 0.0002250 \\
0.01     & $>500$       & 5/5 & 0.3144 (0.0013) & 0.0001000  & 0.0001000 \\
0.005    & $>500$       & 5/5 & 0.2646 (0.0000) & 0.0000250  & 0.0000250 \\
0.001    & $>500$       & 5/5 & 0.2646 (0.0000) & $1.00\times10^{-6}$ & $1.00\times10^{-6}$ \\
$10^{-5}$ & $>500$      & 5/5 & 0.2646 (0.0000) & $1.00\times10^{-10}$ & $1.00\times10^{-10}$ \\
0 (control) & $>500$    & 5/5 & 0.2646 (0.0000) & 0 & 0 \\
\bottomrule
\end{tabular}
\end{table}

\subsection{Replication details for the BERT unlearning experiment}
\label{app:bert-retain-mse}

The scrub trains target-class hidden representations toward independently
sampled retain-feature vectors while optimizing retain-class cross-entropy.
Target accuracy falls to $53.4\pm3.6\%$, with retain accuracy
$96.3\pm1.2\%$. Layerwise logistic-decoder estimates decrease from
$0.433/0.449/0.457$ to $0.178/0.150/0.115$ nats; these are estimates of
predictive $\mathcal V$-information \citep{xu2020usable} for the specified decoder class, not Shannon MI.

The scrubbed model recovers in $252\pm67$ updates across three seeds,
compared with $>1{,}000$ for random initialization. A pretrained control
with no target adaptation or scrub reaches the threshold in $75\pm5$
updates. This establishes that generic pretraining transfer can support
rapid acquisition, but does not apportion the scrub's retained structure
between generic and target-specific contributions. The control is not
utility-matched: its retain accuracy is $34.1\pm15.3\%$, versus
$96.3\pm1.2\%$ for the scrubbed model. ``No target adaptation'' refers to
the controlled AG News stage, not a guarantee about the original BERT
pretraining corpus.

A paired seed-0 orbit sweep isolates parameterization dependence within
the scrubbed arm. At $\alpha=1$ the checkpoint recovers in 280 steps; at
$\alpha=32$ it recovers in 25, using the same attack batches and
learning-rate grid. The exact Q/K symmetry preserves the released
function and the observed hidden representations; the logit invariance
check gives zero maximum deviation on the evaluation batch. This
single-seed comparison is separate from the three-seed arm means and
does not establish complete MI removal or full-kernel preservation.

We use Google BERT-Tiny \citep{devlin2019bert}, revision \texttt{30b0a37ccaaa32f332884b96992754e246e48c5f}, with two transformer blocks, hidden size 128, and maximum sequence length 64. The classifier has a 64-dimensional linear representation layer and a four-class linear head. We adapt the pretrained model for 400 AdamW updates at learning rate $2\times10^{-4}$ on AG News. The target classes are World (class 0) and Business (class 2); retain classes are Sports (class 1) and Sci/Tech (class 3). Each class contributes 1,000 training examples and 500 examples to each probe, validation, and test split. All splits and sampling indices use seed 0.

Let $h_\ell(x)$ denote the masked-mean representation at the embedding output or after encoder block $\ell$, and let $b_\ell$ be a bank of 512 retain-set representations from the same layer. Each update samples 32 target examples and 32 bank indices. The objective is
\[
\mathcal L(\theta)=\frac{1}{3}\sum_{\ell=0}^{2}\left\|h_\ell(x_i)-b_{\ell,j_i}\right\|_2^2+\mathcal L_{\rm CE}(f_\theta(x_R),y_R),
\]
where the retain minibatch contains 32 randomly sampled non-target examples. Both terms have coefficient one. All parameters are optimized with AdamW at learning rate $10^{-3}$, zero weight decay, for 1,000 updates.

For each layer, a fresh logistic decoder is fitted on the probe split, selected by validation cross-entropy, and evaluated on the test split. We estimate predictive $\mathcal V$-information \citep{xu2020usable} as $\widehat I_{\mathcal V}(Y;Z)=\log 2-\mathrm{CE}_{\rm test}$ in nats. The matched random encoder baseline has layerwise values $0.1403$, $0.1392$, and $0.1397$ nats.

The full target-margin NTK uses 32 balanced target test examples. For example $i$, let $J_i=\partial(f_0(x_i)-f_2(x_i))/\partial\theta$ and set $K=JJ^\top$. We report $\operatorname{tr}(K)$ and relative drift $\|K-K_{\rm pre}\|_F/\|K_{\rm pre}\|_F$. Recovery starts from each checkpoint with all parameters trainable and uses a balanced target minibatch sequence shared across learning rates and, in the paired orbit comparison, across scales. We report the smallest hitting time over the fixed SGD learning-rate grid $\{3\times10^{-4},10^{-3},3\times10^{-3}\}$; both paired orbit endpoints attain their best time at $3\times10^{-3}$. The recovery step is the first update with target accuracy at least 75\% and target cross-entropy at most 0.5; runs are stopped at 1,000 updates.

\begin{table}[ht]
\centering
\small
\resizebox{\textwidth}{!}{%
\begin{tabular}{lcccccc}
\toprule
Arm & Target acc. & Retain acc. & $\widehat I_{\mathcal V}$ (emb./L1/L2) & $\operatorname{tr}(K)$ & NTK drift & Recovery \\
\midrule
Target-adapted & 91.7\% & 82.4\% & $0.433/0.449/0.457$ & $9.59\times10^5$ & $0$ & $0$ \\
No target adaptation & $\approx50\%$ & $34.1\pm15.3\%$ & -- & -- & -- & $75\pm5$ \\
MSE scrub + retain & $53.4\pm3.6\%$ & $96.3\pm1.2\%$ & $0.178/0.150/0.115$ & $5.74\times10^4$ & $0.952$ & $252\pm67$ \\
\quad Seed 0, $\alpha=1$ & -- & -- & -- & $4.70\times10^4$ & -- & $280$ \\
\quad Seed 0, $\alpha=32$ & (unchanged) & (unchanged) & (unchanged) & $1.01\times10^5$ & -- & $25$ \\
Random initialization & 50.0\% & 21.7\% & $0.140/0.139/0.140$ & $1.11\times10^4$ & $-$ & $>1000$ \\
\bottomrule
\end{tabular}%
}
\caption{Full results for \Cref{tab:bert-retain-mse}, including retain
accuracy, tangent-kernel trace, and NTK drift. Recovery is ordinary
full-parameter SGD, best of the learning-rate grid
$\{3\times10^{-4},10^{-3},3\times10^{-3}\}$ (three seeds for the
no-target-adaptation and MSE-scrub arms; seed 0 for the paired orbit rows);
random initialization does not reach the recovery criterion within
1,000 updates.}
\label{tab:bert-retain-mse-full}
\end{table}

\paragraph{No-target-adaptation control.} Same pretrained encoder and
the same fresh rep/head initialization as the other arms (captured once,
before any adaptation, and shared across arms), but neither the 400-step
target adaptation nor the scrub is applied. Recovery uses the identical
procedure, minibatch-stream seeding, and threshold as the main runs, over
3 seeds.

\paragraph{Paired Q/K orbit comparison.}
BERT uses separate query and key projections. Applying
$Q\mapsto\alpha Q$, $K\mapsto\alpha^{-1}K$ (including biases) at every
layer preserves attention scores and the represented function. The
seed-0 sweep uses the saved scrub checkpoint, with maximum measured logit
deviation zero at $\alpha=32$. Both endpoints use the same batches,
learning-rate grid, and threshold. The measured target-margin kernel
trace changes from $4.70\times10^4$ at $\alpha=1$ to $1.01\times10^5$ at
$\alpha=32$. The paired hitting times are 280 and 25 steps, respectively.
This demonstrates attack-speed variation within the orbit; it does not
identify all causes of the scrub's slowdown relative to other arms.

\end{document}